\documentclass[10pt,twocolumn]{article}
\usepackage[
left=2.00cm,
right=2.00cm,
top=2.00cm,
bottom=2.00cm
]{geometry}
\pdfoutput=1
\usepackage{titlesec}
\titleformat{\section}[hang]
{\large\bfseries}
{\thesection.}{0.5em}{}

\titleformat{\subsection}[hang]
{\bfseries}
{\thesubsection.}{0.5em}{}

\usepackage{footnote}
\makesavenoteenv{tabular}
\makesavenoteenv{table}

\usepackage{titlesec}
\usepackage{abstract}

\usepackage{color}
\usepackage{bbm}
\usepackage[title]{appendix}
\usepackage{bm}
\usepackage{mathtools}
\usepackage{amsthm}
\usepackage{algorithm, algorithmic}
\usepackage{soul} 
\usepackage{hyperref}
\usepackage{booktabs}
\usepackage{graphicx}
\usepackage{titlesec}
\usepackage{dirtytalk}

\newcommand{\I}{\mathbbm{I}}

\newcommand{\Mhat}{\hat{M}}

\newcommand{\f}[1]{f^{#1}}
\newcommand{\ytilde}{\tilde{y}}
\newcommand{\rhat}{\hat{r}}
\newcommand{\yhat}{\hat{y}}
\newcommand{\ghat}{\hat{g}}
\newcommand{\chat}{\hat{c}}
\newcommand{\norm}[1]{\left\lVert#1\right\rVert}
\newcommand{\Lhat}{\hat{L}}
\newcommand{\Phihat}{\hat{\Phi}}
\newcommand{\Otilde}{\tilde{\mathcal{O}}}
\newcommand{\C}{\mathcal{C}}

\newcommand{\D}{\mathcal{D}}
\newcommand{\Top}{\mathcal{T}}
\newcommand{\X}{\mathcal{X}}
\newcommand{\rtilde}{\tilde{r}}
\newcommand{\R}{\mathbbm{R}}
\newcommand{\E}{\mathbbm{E}}

\newcommand{\Logloss}{L^{\mathrm{log}}}
\newcommand{\hatrankloss}{\hat{L}^{\mathrm{rnk}}}
\newcommand{\lhatlog}{\hat{L}^{\mathrm{log}}}
\newcommand{\Rank}{L^{\mathrm{rnk}}}
\newcommand{\optimalAlgo}{Top-\textit{k}BBM}
\newcommand{\adaptiveAlgo}{Top-\textit{k}Adaptive}
\newcommand{\bigO}{\mathcal{O}}

\newcommand{\hinge}{L^\mathrm{hinge}}
\newcommand{\WRank}{L^\mathrm{wrnk}}

\DeclareMathOperator{\WL}{WL}
\DeclareMathOperator{\WLC}{Top-\textit{k}WLC}
\DeclareMathOperator{\OnlineWLC}{OnlineWLC}

\DeclareMathOperator{\argmin}{argmin}

\usepackage{thmtools,thm-restate}
\newtheorem{theorem}{Theorem}

\newtheorem{definition}{Definition}
\newtheorem{lemma}[theorem]{Lemma}
\newtheorem{assumption}{Assumption}

\theoremstyle{remark}
\newtheorem*{remark}{Remark}

\usepackage[round]{natbib}
\newcount\comments  
\definecolor{darkgreen}{rgb}{0,0.6,0}
\newcommand{\genComment}[2]{\ifnum\comments=1{\color{#1}{\textsf{\footnotesize #2}}}\fi}

\title{Online Boosting for Multilabel Ranking with Top-$k$ Feedback}

\author{
Vinod Raman\thanks{denotes alphabetic ordering indicating equal contributions} \\
University of Michigan\\
\texttt{vkraman@umich.edu}\\
\and
Daniel T. Zhang\footnotemark[1]\\
Facebook Infrastructure Engineering \\ 
\texttt{dtzhang@fb.com}\\
  \and
  Young Hun Jung\\
  Microsoft\\
  \texttt{youjung@microsoft.com} \\
  \and
  Ambuj Tewari\\
  University of Michigan\\
  \texttt{tewaria@umich.edu} \\
}

\begin{document}
\twocolumn[
\maketitle]

\begin{abstract}
\noindent We present online boosting algorithms for multilabel ranking with top-$k$ feedback, where the learner only receives information about the top
$k$ items from the ranking it provides. We propose a novel surrogate
loss function and unbiased estimator, allowing
weak learners to update themselves with limited information.
Using these techniques we adapt full information multilabel
ranking algorithms \citep{jungMultilabel} to the top-$k$ feedback setting
and provide theoretical performance bounds which closely match the bounds of their full information counterparts, with the cost of increased sample complexity. These theoretical results are further substantiated by our experiments, which show a small gap in performance between the algorithms for the top-$k$ feedback setting and that for the full information setting across various datasets. 
\end{abstract}
\saythanks  

\section{INTRODUCTION}

The classical theory of boosting is an impressive algorithmic and theoretical achievement (see \citet{schapire2012boosting} for an authoritative treatment). However, for the most part it assumes that learning occurs with a batch of data that is already collected and that ground truth labels are fully observed by the learning algorithm. Modern ``big data" applications require us to go beyond these assumptions in a number of ways.

First, large volumes of available data mean that {\em online algorithms} \citep{shalev2012online,hazan2016introduction} are needed to process them effectively. Second, in many applications such as text categorization, multimedia (e.g., images and videos) annotation, bioinformatics, and cheminformatics, the ground truth may not be just a single label but a {\em set of labels} \citep{zhang2013review,gibaja2015tutorial}. Third, in a multilabel setting, a common design decision \citep{DATA} is to have the learner output a {\em ranking} of the labels. Fourth, human annotators may not have the patience to go down the full ranking to give us the ground truth label set. Therefore, the learner may have to deal with {\em partial feedback}. A very natural partial feedback is {\em top-$k$ feedback} \citep{chaudhuri2017online} where the annotator only provides ground truth for the top-$k$ ranked labels. Theory and algorithms for online, multilabel boosting with top-$k$ feedback have thus far been missing. Our goal in this paper is to fill this gap in the literature.

Existing literature has dealt with some of the challenges mentioned above. For example, recent work has developed the theory of {\em online boosting} for single label problems such as binary classification \citep{beygelzimer2015optimal} and multiclass classification \citep{jung17online}. This was followed by an extension to the complex label setting of {\em multilabel ranking} \citep{jungMultilabel}. 
All of these works were in the {\em full information} setting where ground truth labels are fully revealed to the learner. \citet{zhangBandit} recently extended the theory of multiclass boosting to the {\em bandit} setting where the learner only observes whether the (single) label it predicted is correct or not. However, none of the available extensions of classical boosting has {\em all} of the following three desired attributes at once, namely online updates, multilabel rankings, and the ability to learn from only top-$k$ feedback.

Note that top-$k$ feedback is {\em not} bandit feedback. Unlike the bandit multiclass setting, the learner does not even get to compute its own loss! Thus, a key challenge in our setting is to use the structure of the loss to design estimators that can produce unbiased estimates of the loss from only top-$k$ feedback. This intricate interplay between loss functions and partial feedback does not occur in previous work on online boosting.

Specifically, we extend the full information algorithms of \citet{jungMultilabel} for multilabel ranking problems to the top-$k$ feedback setting . 
Our algorithms randomize their predictions and construct novel unbiased loss estimates. 
In this way, we can still let our weak learners update themselves even with partial feedback.
Interestingly, the performance bounds of the proposed algorithms match the bounds of their full information counterparts with increased sample complexities. 
That is, even with the top-$k$ feedback, one can eventually obtain the same level of accuracy provided sufficient amount of data.
We also run our algorithms on the same data sets that are investigated by \citet{jungMultilabel}, and obtain results supporting our theoretical findings. 
Our empirical results also verify that a larger $k$ (i.e., more information available to the learner) does decrease the sample complexity of the algorithms.




\section{PRELIMINARIES}
The space of all possible labels is denoted by
$[m] = \{1, \cdots, m\}$ for some positive integer $m$, which we assume is known to the
learner. 
We denote the indicator function as
$\I(\cdot)$, the $i$th standard basis vector
as $e_i$, and the zero vector as $0$.
Let $\D_m = \{e_i | i \in [m]\}$. 
We denote a ranking as an ordered tuple.
For example, a ranking $r = (2, 1, 3)$ ranks label $2$ the highest and label $3$ the lowest.
Given a ranking $r$, we let $\Top^k(r)$ return an unordered set of the $k$ top ranked elements. 
For example, $\Top^2((2, 1, 3)) = \{1, 2\}$.

We will frequently use a \textit{score vector}
$s \in \R^m$ to denote a ranking.
We convert it to a ranking using the function $\sigma(s)$, which
orders the members of $[m]$ according to the score $s[i]$. 
We break ties by preferring smaller labels, for example $3$ is preferred over $4$ if votes are even.
This makes the mapping $\sigma(s)$ injective. 
For example, 
$\sigma((.3, .7, .5, .5)) = (2, 3, 4, 1)$.
When it is clear from the context, we will use a score vector and the corresponding ranking interchangeably.

\subsection{Problem Setting}
We first describe the \textit{multilabel ranking} (MLR) problem with top-$k$ feedback.
At each timestep $t=1, \cdots, T$, the relevant labels are $R_t \subseteq [m]$,
and the irrelevant labels are $R_t^c$. 
An adversary chooses a labelled example $(x_t, R_t) \in \X \times 2^{[m]}$ 
(where $\X$ is some
domain) and sends $x_t$ to the learner. 
As we are interested in the MLR setting, the learner then produces an $m$-dimensional score vector $y_t$, and sends this result to the adversary.
In the full information setting, the learner then observes $R_t$ and suffers a loss $L(y_t, R_t)$ which will be defined later.
In the top-$k$ feedback setting, however, it only observes whether $l$ is in $R_t$ for each label $l \in \Top^k(\sigma(y_t))$. 
That is to say, if $k = m$, then it becomes the full information problem, and smaller $k$ implies less information. 
This feedback occurs naturally in applications such as ads placement and information retrieval, where the end user of the system has limited feedback capabilities. 
In such scenarios, $R_t$ may be the set of ads or documents which the user finds relevant, and $[m]$ may be the total set of documents. 
The user only gives feedback (e.g., by clicking relevant ads) for a few documents placed on top by the algorithm. 
The learner's end goal is still to minimize the loss $L(y_t, R_t)$. 
It might not be able to compute the exact value of the loss because $R_t$ is unknown. 
We want to emphasize that even the size of relevant labels $|R_t|$ is not revealed.

To tackle this problem we use the online multilabel boosting setup of 
\cite{jungMultilabel}. In this setting, the learner
is constructed from $N$ online weak learners, $\WL^1, \cdots, \WL^N$ plus a booster
which manages the weak learners.
Each weak learner predicts a
probability distribution across all possible labels, which we
write as $h^i_t \in \R^m$. Previous work has shown that this weak learner restriction encompasses a variety of prediction formats including binary predictions, multiclass single-label predictions, and multiclass multilabel predictions \citep{jungMultilabel}.

In the MLR version of boosting, each round starts when the booster receives $x_t$. It shares this with all the weak learners and
then aggregates their predictions into the final score vector $y_t$.
Once the
booster receives its feedback, it constructs a cost vector
$c^i_t \in \R^m$ for $\WL^i$ so that the weak learners incur loss 
$c^i_t \cdot h^i_t$,
where $h^i_t$ is the $i$th weak learner's prediction at time $t$.
Each weak learner's
goal is to adjust itself over time to produce $h^i_t$ that minimizes its loss.
The goal 
of the booster is to 
generate cost vectors which encourage the weak learners to cooperate in creating better $y_t$.
It should be noted that despite the top-$k$ feedback, our weak learners get full-information feedback, 
which means the entire cost vector is revealed to them. 
Constructing a complete cost vector with partial feedback is one of the main challenges in this problem.

\subsection{Estimating a Loss Function}
\label{sec:unb_est}
Because of top-$k$ feedback, we require methods to estimate loss functions dependent on labels outside of the top-$k$ labels from our score vector $y_t$. 
One common way of dealing with partial feedback is to introduce randomized predictions and construct 
an unbiased estimator of the loss using the known distribution of the prediction. 
This way, we can obtain a randomized loss function 
for our learner to use. 
Thus, we propose a novel unbiased estimator to randomize arbitrary $y_t$. This estimator requires some structure within the loss function it is estimating.

We require that the loss be expressible as a sum of functions which only require as input the scores and relevance of two particular labels, each containing one relevant and irrelevant label. In particular, our loss must have the form
\begin{align*}
L(s, R) 
= 
\sum_{a \in R} \sum_{b \notin R}
f(s[a], s[b]) 
=:
\sum_{a, b \in [m]}
\f{a, b}(s),
\\
\text{where }
\f{a,b}(s)= 
\I(a \in R) \I(b \notin R)f(s[a], s[b]).
\end{align*}
Here $s$ is an arbitrary score vector in $\R^m$, and $f$ is a given function.
We call this property \textit{pairwise decomposability}.
This decomposability allows us to individually estimate each $\f{a, b}$ and thus $L$.

In fact, various valid MLR loss functions are pairwise decomposable.
An example is the \textit{unweighted rank loss}
$
\Rank(s, R_t) = \sum_{a \in R_t} \sum_{b \notin R_t} \I(s[a] \leq s[b]),
$
which has various surrogates, including the following \textit{unweighted hinge rank loss}
$
\hinge(s, R_t) = \sum_{a \in R_t} \sum_{b \notin R_t} \max\{0, s[b] - s[a]+1\}.
$
It should be noted that the \textit{weighted rank loss}
$
\WRank(s, R_t) = \frac{1}{|R_t|(m - |R_t|)}\Rank(s, R_t)
$
cannot be computed using this strategy because its normalization weight
is non-linearly dependent on $|R_t|$. 
In such cases, it is possible to upper bound the target loss function with a surrogate loss that is pairwise decomposable. 
For example, the unweighted rank loss is an obvious upper bound of the weighted one.


Returning to our estimator, we first describe a simple method of randomized prediction given $r_t = \sigma(y_t)$ that will allow us to construct our unbiased estimator.
This randomized prediction $\rtilde_t$  is paramaterized by the exploration rate $\rho \in (0, 1)$. 
After computing $r_t$, with probability $1-\rho$ we use $r_t$ as our final ranking. Otherwise, with probability $\rho$, we choose a uniform random permutation of the labels denoted as $\rtilde_t$. In case the loss is a function of score vector instead of ranking, we can get a random score vector $\tilde{s}_t$ out of $s_t$ in a similar manner by shuffling the scores uniformly at random and breaking any ties randomly.

We note that there are no canonical randomization protocols for ranking problems: many different randomization schemes can be applied but the theoretical loss bound will depend on which randomization procedure is selected. That is, one randomization procedure can be deemed \say{better} than another in the sense that it improves the bound on rank loss. However, the pursuit of improving theoretical bounds can lead to rather complex randomization schemes that impose unnatural restrictions on feasible values of $m$ and $k$. Thus, there exists a trade off between the simplicity/flexibility of a randomization scheme and its impact on the loss bounds. To this end, we investigated several randomization candidates both theoretically and empirically, and here we presented the uniform randomization scheme due to its simplicity and comparable empirical performance to more sophisticated randomization procedures. Other schemes can be found in Appendix \ref{sec:alt_rand}.

We now present our unbiased estimator. Let $\rtilde_t$ be the random ranking from the previously described process, and let $s$ be an arbitrary score vector in $\R^m$. We note that given any two distinct labels $a$ and $b$,
$\Pr[a, b \in \Top^k(\rtilde_t)] > 0$. Since being in the top-$k$ provides the learner with full information regarding the relevance and scores of the labels, we have this unbiased estimator using importance sampling
\begin{equation}
\label{unb_est}
\Lhat(s, R_t) = \sum_{a,b \in [m]}
\frac{\I(a, b \in \Top^k(\rtilde_t))}{\Pr[a, b \in \Top^k(\rtilde_t)]}
\f{a, b}(s).
\end{equation}
We prove that this is an unbiased estimator in Lemma \ref{unb_est_proof} in the appendix. 
Our algorithms will use this unbiased estimator to estimate certain surrogate functions which we construct to be pairwise decomposable.

One useful quality of this estimator is what we
call \textit{$b$-boundedness}. We say a random vector $Y$ is
$b$-bounded if almost surely, $\norm{Y - \E Y}_\infty \leq b$.
This definition also applies to scalar random
variables, in which case the infinity norm becomes the absolute value.
In Lemma \ref{unb_est_bound} in the appendix, we prove that if
the pairwise functions are bounded by some  $z$, then
any such unbiased estimator like in Eq. \ref{unb_est} is bounded with a constant that is
$\mathcal{O}(z\frac{2m^2}{\rho})$. We note that this bound suppresses the dependence on $k$, and in fact, increasing $k$ improves the $b$-boundedness of the estimator. This behavior is substantiated by our empirical results. 



Now suppose that the cost vector $c^i_t$ (to be fed to the $i$th weak learner at time $t$) requires full knowledge of $R_t$ to compute.
If each of its entries is a function that is pairwise decomposable, we can use the same unbiased estimation strategy to obtain random cost vectors $\chat^i_t$ that are in expectation equal to $c^i_t$.

\section{ALGORITHMS}
We introduce two different online boosting algorithms along with their performance bounds. 
Our bounds rely on the number and
quality of the weak learners, so we define the edge of a weak learner. 
Our first algorithm assumes every weak learner
has a positive edge $\gamma$, while our second algorithm uses an edge measured adaptively.



\newcounter{init}
\newcounter{index}
\newcounter{cost}
\newcounter{alpha}
\newcounter{booster}

\subsection{Algorithm Template}
Here, we present the template which our two boosting algorithms
share. 
\begin{algorithm}[ht]
	\begin{algorithmic}[1]
	    \STATE \textbf{Input:} Exploration rate $\rho$ and a loss $L(\cdot, \cdot)$ where $L$ is pairwise decomposable
	    \setcounter{init}{\value{ALC@line}}
	    \STATE \textbf{Initialize:} $\WL$ weights $\alpha^{i}_{1}$ for $i \in [N]$
		\FOR {$t = 1, \cdots, T$}
		\STATE Receive example $x_{t}$
		\STATE Obtain distribution $h^{i}_{t}$ from $WL^i$ for $i \in [N]$
		\STATE Compute experts $s^j_t = \sum_{i=1}^j \alpha^i_t h^i_t$ for $j \in [N]$
		\setcounter{index}{\value{ALC@line}}
		\STATE Select an index $i_t \in [N]$ and set $\yhat_t = s^{i_t}_t$
		\STATE Obtain $\rtilde_t$ from $\sigma(\ytilde_t)$ using $\yhat_t$ and the random process defined in Section \ref{sec:unb_est}
		\STATE Booster suffers loss $L(\ytilde_t, R_t)$, but this is not shown to the booster
		\STATE For each $l \in \Top^k(\rtilde_t)$, receive feedback $\I(l \in R_t)$
		\setcounter{cost}{\value{ALC@line}}
		\STATE Compute cost vectors $\chat^{i}_{t}$ for each $i \in [N]$ 
		\STATE Weak learners suffer loss $\chat^{i}_{t} \cdot h^i_t$ and update internal parameters
		\setcounter{alpha}{\value{ALC@line}}
		\STATE Set $\alpha^i_{t+1}$ for all $i \in [N]$
		\setcounter{booster}{\value{ALC@line}}
		\STATE Update booster's parameters, if any
		\ENDFOR
	\end{algorithmic}
	\caption{Online Rank-Boosting Template}
	\label{alg:template}
\end{algorithm}
It does not specify certain steps, which will
be filled in by the two boosting algorithms. Also,
in our template we do not restrict weak learners in any way except
that each $\WL^i$ predicts a distribution $h^i_t$ over $[m]$,
receives a full cost vector $\chat^i_t \in \R^m$, 
and suffers the loss $\chat^i_t \cdot h^i_t$ according
to its prediction. This allows our boosting framework to be general enough to incorporate various types of weak learning predictions such as single-label and multilabel learners. 


The booster keeps updating the learner weights
$\alpha^i_t$ and constructs $N$ experts, where the
$j$th expert is the weighted cumulative votes from the
first $j$ weak learners
$s^j_t \coloneqq \sum_{i=1}^j \alpha^i_t h^i_t \in \R^m$.
The booster chooses an expert index
$i_t \in [N]$ at each round $t$ to use.
The first algorithm fixes $i_t$ to be $N$, while the second one draws it randomly using an adaptive distribution.
The booster then uses $s^{i_t}_t$ to compute its final random prediction
$\ytilde_t$. After obtaining
feedback,
the booster computes random cost vectors $\chat^i_t$ for
each weak learner and lets them update parameters.

\subsection{An Optimal Algorithm}
Our first algorithm, \optimalAlgo\footnote{ Boost-By-Majority for ranking with top-$k$ feedback},
assumes the ranking weak learning condition and
is optimal, meaning it matches the asymptotic loss bounds
of an optimal full information boosting algorithm in the number
of weak learners used, up to a constant factor.

\subsubsection{Ranking Weak Learning Condition}

The ranking weak learning condition
states that within the cost vector framework,
weak learners can minimize their
loss better than a randomly guessing competitor, so long as the
cost vectors satisfy certain conditions, and with the
weak learners only observing versions of the cost vectors tainted by some noise.

We define the randomly guessing competitor at time $t$ as
$u^\gamma_{R_t}$, which
is an almost uniform distribution placing $\gamma$ more weight
on each label in $R_t$. In particular, for a any label $l$ we define it as
\begin{align*}
u^\gamma_{R_t}[l] = 
\begin{cases}
\frac{1-|R_t|\gamma}{m}  + \gamma 
&\text{ if } l \in R_t \\
\frac{1-|R_t|\gamma}{m}
&\text{ if } l \notin R_t
\end{cases}.
\end{align*}
The intuition is that if a weak
learner predicts a label using $u^\gamma_{R_t}$ at each round,
then its accuracy would be better than random guessing by at
least an edge of $\gamma$.

Given $R_t$, we specify the set of possible cost vectors as
\begin{align*}
\C(R_t) = \{ c \in [0, 1]^m &\mid \min_l c[l] = 0, \max_l c[l] = 1, \\
&\max_{i \in R_t}c[i] \le 
\min_{j \notin R_t}c[j] \}.
\end{align*}
We also allow a sample weight $w_t \in [0, 1]$ to be multiplied by this cost vector.
This feasible set of cost vectors is equivalent to those
used in the full information setting studied by \cite{jungMultilabel}.

As in the full information setting, at each round we allow the adversary to choose an arbitrary cost vector from $\mathcal{C}(R_t)$ and its weight for the learner. In our top-$k$ feedback setting, we further permit the adversary to generate random cost vectors and weights, so long as in expectation each random cost vector is in $\mathcal{C}(R_t)$.

We now introduce our top-$k$ feedback weak learning condition, presented beside the full information online weak learning condition from \cite{jungMultilabel}, to show their similarity.

\begin{definition}[OnlineWLC]
For parameters $\gamma, \delta$, and $S$, a pair of a learner and adversary satisfies $\OnlineWLC(\gamma, \delta, S)$ if for any $T$, with probability $1-\delta$, the learner can generate predictions that satisfy
$$
\sum_{t=1}^T w_t c_t \cdot h_t \leq \sum_{t=1}^T w_t c_t \cdot u^\gamma_{R_t} + S.
$$
\end{definition}

\begin{definition}[$\WLC$]
For parameters $\gamma, \delta, b$, and $S$, a pair of a learner and adversary satisfies $\WLC(\gamma, \delta, S, b)$ if for any $T$, with probability $1-\delta$, the learner can generate predictions that satisfy
$$
\sum_{t=1}^T w_t c_t \cdot h_t \leq \sum_{t=1}^T w_t c_t \cdot u^\gamma_{R_t} + S,
$$
while only observing random cost vectors $\chat_t$, where all $\chat_t$ are $b$-bounded  and $\E \chat_t = w_t c_t$.
\end{definition}
In these definitions, $S$ is called the {\em excess loss}. 
The two weak learning conditions differ only by the introduction of random noise. In $\WLC$ if the variance of each $\chat_t$ is $0$, 
then $\chat_t = c_t$, and we recover the full information weak learning condition.
For a positive $b$, the definition of $b$-boundedness implies $\frac{\chat}{b}$ is $1$-bounded. 
From this, we can infer $S = \bigO(b)$ in the top-$k$ setting. 


\subsubsection{\optimalAlgo~Details}
\label{sec:potential}
\begin{algorithm}[ht]
	\begin{algorithmic}[1]
		\setcounter{ALC@line}{\value{init}}
		\STATE \textbf{Initialize:} $\WL$ weights $\alpha^{i}_{1}=1$ for $i \in [N]$
		\setcounter{ALC@line}{\value{index}}
		\STATE Set $i_t = N$
		\setcounter{ALC@line}{\value{cost}}
		\STATE Compute cost vectors $\chat^{i}_{t}$ for each $i$ using Eq. \ref{cost_vect}
		\setcounter{ALC@line}{\value{alpha}}
		\STATE Set $\alpha^i_{t+1} = 1$
		\setcounter{ALC@line}{\value{booster}}
		\STATE No booster parameters to update
	\end{algorithmic}
	\caption{\optimalAlgo}
	\label{alg:optimal}
\end{algorithm}
We require that our loss function $L(s, R_t)$ be pairwise decomposable, and that each of its pairwise function $f$ has three properties, which we now define. 
\begin{assumption}[Properness]\label{as:prop}
  $f(s[a], s[b])$ is non-increasing in $s[a]$ for $a \in R_t$ and non-decreasing in $s[b]$ for $b \notin R_t$. Intuitively, properness requires that putting a higher score on a relevant label should decrease the loss. 
\end{assumption}
\begin{assumption} [Uncrossability]\label{as:uncross}
  For any $a \in R_t$, $b \notin R_t$, and $\beta >0$, we assume $f(s[a]+\beta, s[b]+\beta) \geq f(s[a], s[b])$. Intuitively, this means that if a weak learner is unsure which label to prefer, it cannot place even weight on both labels and cheat its way to a lower cost.
\end{assumption}
\begin{assumption} [Convexity]\label{as:convex}
    For any $a \in R_t$ and $b \notin R_t$, $f(s[a], s[b])$ is convex w.r.t. the score vector $s$.
\end{assumption}

    
    

We prove these assumptions for the loss functions we use in Appendix 
\ref{loss_properties_proof}. We now briefly discuss notation to describe potential functions,
whose relation to boosting has been thoroughly discussed by \cite{mukherjee}.
Let $\D_m = \{e_i | i \in [m]\}$ and let $u$ be a distribution over this set. Given a starting vector $s \in \R^m$,
a function $f: \R^m \mapsto \R$, and a non-negative integer $N$, we define 
$\varphi^{N-i}_u(s, f) = \E f(s + X)$ where $X$ is the summation of $N-i$ random vectors drawn from $u$.

Moving on to potential functions for boosting, in the full information setting where $R_t$ is revealed
we would use the ground truth potential function
$
\Upsilon^{N-i}_{t}(s) \coloneqq \varphi^{N-i}_{u^\gamma_{R_t}}(s, L(\cdot, R_t))
$ 
to create cost vectors.
It takes the current cumulative
votes $s \in \R^m$ as an input and estimates the booster's loss when the relevant labels are $R_t$ and the weak
learners guess from a distribution $u^\gamma_{R_t}$. 
However, because $R_t$ is not known in our setting, we provide a surrogate
potential function, using the assumptions listed previously in this section.

To compute our surrogate potential function, we first rewrite the ground truth potential
by moving the expectation inside
the pairwise summations:
\begin{equation}
\label{gt_pair}
\begin{gathered}
\Upsilon^{N-i}_t(s) = \sum_{a, b \in [m]} 
	\varphi^{N-i}_{u^\gamma_{R_t}}(s, \f{a, b}),
\end{gathered}
\end{equation}
where we slightly abuse the notation by letting $\f{a, b}$ takes $s$ as an input.
Then we propose the following surrogate potential function
\begin{gather*}
\Phi^{N-i}_t(s) = \sum_{a \in R_t} \sum_{b \notin R_t}
    \Lambda^{a, b, N-i}_{t}(s), \text{ where}\\
\Lambda^{a, b, N-i}_{t}(s) =
	\varphi^{N-i}_{u^\gamma_a}(s, \f{a, b}) 
	\text{ with } 
	u^\gamma_a = u^\gamma_{\{a\}}.
\end{gather*}
We record important qualities of this surrogate potential as a proposition, with the proof in Appendix \ref{sec:optimal_proofs}.

\begin{restatable}{propoosition}{primelemma}
\label{mylemma}
Given Assumptions \ref{as:prop}, \ref{as:uncross}, and \ref{as:convex} on $f^{a,b}$,
$\Phi^{N-i}_t(s)$ is proper and convex, and for any $R$, $\gamma$, $N$, and $s$, we have
$
\Upsilon^{N-i}_t(s) \leq \Phi^{N-i}_t(s).
$
\end{restatable}

We also stress that $\Phi^{N-i}_t$ is pairwise decomposable into each of its smaller potential functions.

Returning to the algorithm, 
we assume that weak learners satisfy $\WLC(\gamma, \delta, S, b)$.
Our goal is to set 
$c^i_t[l] = \Phi^{N-i}_{t}(s^{i-1}_t + e_l)$.
Because $\Phi$ is pairwise decomposable, we can create an unbiased estimator
of it using the technique in Section \ref{sec:unb_est} as
\begin{align*}
\Phihat^{N-i}_{t}&(s) = \sum_{a \in R_t} \sum_{b \notin R_t} 
	\frac{\I(a, b \in \Top^k(\rtilde_t))}{\Pr[a, b \in \Top^k(\rtilde_t)]}\Lambda^{a, b, N-i}_{t}(s).
\end{align*}
Because $\Lambda^{a, b, N-i}_t$ is simply a potential function using $\f{a,b}$, any upper bound on $\f{a, b}$ also upper bounds $\Lambda^{a, b, N-i}_t$. 
Then we can use Lemma \ref{unb_est_bound} in the appendix to claim $\Phihat^{N-i}_t$ is $b$-bounded, and therefore create unbiased estimates of $c^i_t[l]$ as
\begin{gather}
\label{cost_vect}
\chat^i_t[l] = \Phihat^{N-i}_{t}(s^{i-1}_t + e_l).
\end{gather}

The rest of the algorithm is straightforward. We set $\alpha^i_t = 1$
for all $i \in [N]$, and select the best expert to be $i_t = N$. This
means that we always take an
equal-weighted vote from all the weak learners. Intuitively,
the booster wants to use all weak learners because they are all guaranteed
to do better than random guessing in the long run, and weigh them equally because all
weak learners have the same edge $\gamma$. Lastly, given the last expert $s^N_t$, our algorithm
predicts using the same random process described in Section \ref{sec:unb_est}.

\subsubsection{\optimalAlgo~Loss Bound}

We can theoretically guarantee the performance of \optimalAlgo~on any proper and pairwise decomposable loss function. 
In our theorem, we bound $L(s^N_t, R_t)$ instead of
the true, randomized loss $L(\ytilde_t, R_t)$, 
and then provide corollaries later for specific losses which we will bound $L(\ytilde_t, R_t)$.
The following theorem holds for any pairwise decomposable loss functions whose pairwise losses satisfy the three qualities listed earlier. The proof appears in Appendix \ref{opt_proof}.

\begin{restatable}[\optimalAlgo, General Loss Bound]{theorem}{bbmbound}
\label{bbm_bound}
For any $T, N$ satisfying $\delta \ll \frac{1}{N}$, and any decomposed function $f^{a,b}$ that satisfies Assumptions \ref{as:prop}, \ref{as:uncross}, and \ref{as:convex}, 
the total
loss
incurred by \optimalAlgo~satisfies the following inequality with probability
at least $1-N\delta$
$$
\sum_{t=1}^T L(s^N_t, R_t) \leq \Phi^N_{R_t}(0)T + \Otilde(\frac{2m^2}{\rho}zN),
$$
where $z$ is the maximum possible value that any $\f{a, b}$ can output, and $\Otilde$ suppresses dependence on $\log\frac{1}{\delta}$.
\end{restatable}


Note that there is no single canonical loss in the MLR setting unlike the classification setting where the $0$-$1$ loss is quite standard. Still, the weighted rank loss comes close to being canonical since it is often used in practice and is implemented in standard MLR libraries \citep{DATA}. It has also has been analyzed theoretically in previous work on ranking (e.g., see \cite{rankloss1} and \cite{rankloss2}).

We note that this loss is not convex, and not pairwise decomposable. Thus we use the {\em unweighted} hinge loss as a surrogate. Since the unweighted hinge loss upper bounds the rank loss, Theorem \ref{bbm_bound} can be used to bound it. This allows us to present the following corollary, whose proof can be found in Appendix \ref{opt_corollary_proof}


\begin{restatable}[\optimalAlgo, Rank Loss Bound]{corollary}{optcorollary}
\label{opt_corollary}
For any $T$ and $N \ll \frac{1}{\delta}$, \optimalAlgo's randomized predictions $\ytilde_t$ satisfy the following  bound on the rank loss with probability at least $1-N\delta$


\begin{align*}
\sum_{t=1}^T \Rank_t(\ytilde_t) \leq& \frac{m^2}{4}(N+1)\exp(-\frac{\gamma^2N}{2})T + \rho m^2 T \\ +& \Otilde(\frac{2m^2}{\rho}N^2\sqrt{T}).
\end{align*}
 
\end{restatable}



 

We can optimize 
$\rho\propto NT^{-\frac{1}{4}}\sqrt{2}$
so that the first term in the bound becomes the asymptotic average loss bound.
We can compare it
to the asymptotic error bounds in \cite{jungMultilabel} by multiplying
the full information algorithm loss bounds by $\frac{m^2}{4}$, which is the maximum value 
of the rank loss normalization constant. Let $s'_t$ be the score vectors produced by the full information algorithm.
Then we have that
$$
\sum_{t=1}^T \Rank_t(s'_t) \leq \frac{m^2}{4}(N+1)\exp(-\frac{\gamma^2N}{2})T
	+ \frac{m^2}{2}NS.
$$
We see that the asymptotic losses, after optimizing $\rho$, are identical, so that the cost of top-$k$ feedback
appears only in the excess loss. 
Furthermore, since the optimal algorithm in
\cite{jungMultilabel} is optimal in the number of weak learners it
requires to achieve some asymptotic loss, \optimalAlgo~is also optimal
in this regard since the problem it faces is only harder because of partial feedback.

\subsection{An Adaptive Algorithm}
While \optimalAlgo~is theoretically sound, it has a number of drawbacks
in real world applications. First, it is difficult to actually
measure $\gamma$ for a particular weak learner, and usually the
weak learners will not all have the same edge. Second, potential
functions often do not have closed form definitions, and thus
require expensive random walks to compute. 
To address these issues,
we propose an adaptive algorithm, \adaptiveAlgo, modifying Ada.OLMR from \cite{jungMultilabel}
so that it can use top-$k$ feedback.

\subsubsection{Logistic Loss and Empirical Edges}
Like other adaptive boosting algorithms, we require a surrogate loss.
We take the logistic loss for multilabel ranking from Ada.OLMR, but
ignore its normalization (as that would require knowledge of $|R_t|)$
$$
\Logloss(s, R_t) \coloneqq \sum_{a \in R_t} \sum_{b \notin R_t}
	\log(1+\exp(s[b] - s[a])).
$$
This loss is proper and convex. 
As in Ada.OLMR, the booster's prediction
is still graded using the (unweighted) rank loss. This surrogate loss only plays a role
in optimizing parameters.

Similarly to $\Phihat$, we create an unbiased estimator, $\lhatlog(s, R_t)$,
of the logistic loss
as
\begin{align*}
\sum_{a \in R_t} \sum_{b \notin R_t}
	\frac{\I(a, b \in \Top^k(\rtilde_t))}{\Pr[a, b \in \Top^k(\rtilde_t)]}
	\log(1+\exp(s[b] - s[a])).
\end{align*}
Our goal is to set $c^i_t = \nabla \Logloss_t(s^{i-1}_t)$.
However, because we cannot always evaluate the logistic loss, we use
$\lhatlog(\cdot)$ instead to make random cost vectors which in
expectation are the desired cost vectors:
\begin{gather}
\label{cost_vect_ada}
\chat^i_t = \nabla\lhatlog_t(s^{i-1}_t).
\end{gather}

Even though the algorithm is adaptive, we still need an
empirical measure of the weak learner's predictive powers
for performance bounds.
As in Ada.OLMR, we use the following \textit{empirical edge}
of $\WL^i$:
\begin{equation}
\label{edge_define}
\begin{gathered}
\gamma_i = -\frac{\sum_{t=1}^T c^i_t h^i_t}
		{\norm{w^i}_1} \\
w^i[t] = \sum_{a \in R_t} \sum_{b \notin R_t}
	\frac{1}{1+\exp(s^{i-1}_t[a] - s^{i-1}_t[b])},
\end{gathered}
\end{equation}
where $w^i_t$ is the definition of the \textit{weight} of a 
cost vector taken from \cite{jungMultilabel}. 


\subsubsection{\adaptiveAlgo~Details}
\begin{algorithm}[ht]
	\begin{algorithmic}[1]
		\setcounter{ALC@line}{\value{init}}
		\STATE \textbf{Initialize:} $\WL$ weights $\alpha^{i}_{1}=0$\\
		~~~~and expert weights $\nu^i_t=1$ for $i \in [N]$
		\setcounter{ALC@line}{\value{index}}
		\STATE Select an index $i_t \in [N]$ with $\Pr[i_i=i] \propto \nu^i_t$
		\setcounter{ALC@line}{\value{cost}}
		\STATE Compute cost vectors $\chat^{i}_{t}$ for each $i$ using Eq. \ref{cost_vect_ada}
		\setcounter{ALC@line}{\value{alpha}}
        \STATE Set $\alpha^i_{t+1} = \Pi(\alpha^i_t - \eta_t \ghat^{i'}_t(\alpha^i_t))$ where $\eta_{t} = \frac{8 \rho \sqrt{2}}{m^2 \sqrt{t}}$
		\setcounter{ALC@line}{\value{booster}}
		\STATE Set $\nu^i_{t+1} = \nu^i_t \cdot \exp(-\Lhat^{rnk}(s^{i}_t))$
	\end{algorithmic}
	\caption{\adaptiveAlgo}
	\label{alg:adaptive}
\end{algorithm}
We now go into the details of \adaptiveAlgo. 
The choice of cost
vectors $\chat^i_t$ is discussed in the previous section.
As this is an adaptive algorithm, we want to choose the weak learner's 
weights $\alpha^i_t$ at each round. 
We would like to choose them to minimize the
cumulative logistic loss
$
\sum_t g^i_t(\alpha^i_t) \text{ where }
g^i_t(\alpha^i_t) = \Logloss_t(s^{i-1}_t + \alpha^i_t h_t^i)
$
with only the unbiased estimate
$\sum_t \ghat^i_t(\alpha^i_t) \text{ where }
\ghat^i_t(\alpha^i_t) = \lhatlog_t(s^{i-1}_t + \alpha^i_t h_t^i)$
at each time step available to the booster.

Since the logistic loss is convex, we can use our partial feedback to run
\textit{stochastic gradient descent} (SGD). 
To apply SGD,
besides convexity we require that the feasible space be compact, so
we let $F = [-2, 2]$. To stay in the feasible space,
we use the projection function $\Pi(\cdot) = \max\{-2, \min\{\cdot, 2\}\}$
in our update rule as
$
\alpha^i_{t+1} = \Pi(\alpha^i_t - \eta_t \ghat^{i'}_t(\alpha^i_t)),
$
where $\eta_t$ is the learning rate. We bound the loss from SGD and show it provides a regret within 
$\Otilde(\frac{2m^2}{\rho}\sqrt{T})$. 
The details are in the proof in Appendix \ref{sec:adaptiveProof}.

We cannot prove that the last expert is the best because our weak learners
do not adhere to any weak learning condition. 
Instead, we
prove that at least one expert is reliable. Our algorithm uses the 
\textit{Hedge algorithm} from \cite{hedge} to select the best expert from the
ones available, taking as input for the $i$th expert the unweighted rank loss of the $i$th expert, which we define as
$$
\hatrankloss(s^i_t, R_t) = \sum_{a \in R_t} \sum_{b \notin R_t} \frac{\I(a, b \in \Top^k(\rtilde_t))}{\Pr[a, b \in \Top^k(\rtilde_t]}\I(s[b] \geq s[a]).
$$
Because
the exploration rate $\rho$ controls the variance of the loss estimate, we
can combine the analysis of the Hedge algorithm with a concentration
inequality to obtain a similar result.

\subsubsection{\adaptiveAlgo~Loss Bound}
We now bound the cumulative rank loss of \adaptiveAlgo~using the weak learner's
empirical edges. 
The proof appears in Appendix \ref{sec:adaptiveProof}.

\begin{restatable}[\adaptiveAlgo, Rank Loss Bound]{theorem}{adabound}
\label{ada_bound}
For any $T, N$ satisfying $\delta \ll \frac{1}{N}$, the cumulative rank loss
of \adaptiveAlgo, $\sum_{t=1}^{T}\Rank(\ytilde_t, R_t)$, satisfies the following bound with probability at least
$1-(N+4)\delta$:
$$
 \frac{2m^2}{\sum_i|\gamma_i|}T +\rho m^2T
	+ \Otilde(\frac{2m^2 N\sqrt{T}}{\rho\sum_i|\gamma_i|}),
$$
where $\Otilde$ suppresses dependence on $\log\frac{1}{\delta}$.
\end{restatable}


By optimizing 
$\rho\propto\sqrt{\frac{2N}{\sum_i |\gamma_i|}}T^{-\frac{1}{4}}$, 
we get the first term of the bound as the asymptotic average loss bound.
To compare it with the adaptive algorithm in \cite{jungMultilabel}, we again multiply the bound by $\frac{m^2}{4}$ to account for the normalization constant.
Let $s'_t$ be the scores of the full information adaptive algorithm at time $t$. Then we have
$$
\sum_t \Rank_t(s'_t) \leq \frac{2m^2}{\sum_i |\gamma_i|}T + \Otilde(\frac{N^2m^2}{\sum_i |\gamma_i|}).
$$
This matches the asymptotic loss of Ada.OLMR in \cite{jungMultilabel}, after optimizing for $\rho$. Thus, the cost of top-$k$ feedback is again only present in the excess loss.

\section{EXPERIMENTS}
We compare various boosting algorithms on benchmark data sets using publicly available codes\footnote{We will provide a link to a github in the camera ready version. We are hiding it for review to preserve anonymity.}. The models we use are our own, \optimalAlgo~(TopOpt) and \adaptiveAlgo~(TopAda), along with OnlineBMR (FullOpt) and Ada.OLMR (FullAda) by \cite{jungMultilabel}, the  full information algorithms we compared our theoretical results to. All of these boosters use the same multilabel weak learners as in \cite{jungMultilabel}.



We examine several data sets from the UCI data
repository from \cite{DATA} that have been used to evaluate the full information algorithms.
We follow the preprocessing steps from \cite{jungMultilabel} to ensure consistent comparisons, and the data set details and statistics appear in Table \ref{tab:dataset_summary} in the appendix.
However, because the top-$k$ feedback algorithms require more data to converge, we loop over the training set a number of times before evaluating on the testing data set. 
We consider the number of loops a hyper-parameter. However, it never exceeds $20$. The other hyper-parameters we optimize are the number of weak learners $N$, and the edge $\gamma$ for TopOpt.
Experimental details can be found in Appendix \ref{sec:experiment_details}.

\subsection{Asymptotic Performance}

Since for the rank loss, the theoretical asymptotic error bounds of the proposed algorithms match their full information counterparts, we first compare the models’ empirical asymptotic performance. For the full information algorithms, we looped the training set once and then ran the test set, while for the top-$k$ algorithms, we looped them as described in the previous subsection.
In these tests, we set $k=3$.
The selected hyper-parameters, including number of loops, appear in Table \ref{tab:hyperparameters} in the appendix. 
Each table entry is the result of $10$ runs averaged together. Note, that the average loss of the optimal algorithm on the Mediamill dataset could not be computed in reasonable time due to the large number of labels. The main bottleneck is the computation of potentials as they do not have closed form. To combat this issue, we generated and ran our algorithms on a reduced version titled M-reduced. 

\begin{table}[h]
\caption{Average Weighted Rank Loss on Test Set using Uniform Randomization}
\begin{center}
    \setlength\tabcolsep{2pt}
    \begin{tabular}{cccccccc}
    Data &$m$  &TopOpt &TopAda &FullOpt &FullAda \\
    \hline \\
    Emotions            &6       &0.20  &0.22   &0.17   &0.16\\
    Scene               &6      &0.11   &0.13   &0.07   &0.09\\
    Yeast               &14      &0.23   &0.23   &0.18   &0.19\\
    Mediamill           &101     &--   &0.09   &--   &0.05\\
    M-reduced   &101       &0.11   &0.11   &0.06   &0.06\\
    \end{tabular}
    \label{asymptotic_table}
\end{center}
\end{table}
In Table~\ref{asymptotic_table}, we see that in each data set, the full information algorithms outperform their top-$k$ feedback counterparts, but that the gap is quite small. The largest gap is between TopAda and FullAda on Emotions at $0.06$. 
The smallness of this gap shows our algorithms are learning as effectively as their full information counterparts. Note that the exploration rate $\rho$ for each experiment was tuned and is included in Appendix \ref{sec:experiment_details}.


Another factor is the number of loops run. For the data sets with smaller $m$, the number of loops multiplied by $k$ implies that our algorithms could have observed each label multiple times in its training. For example, with $m=6$ and $k=3$, theoretically in two loops of the training set, our algorithms could have observed them in their entirety.
However, in the Mediamill and M-reduced data sets, our algorithms manage comparable asymptotic performance, while at best, given $m$ and $k$ they could only have possibly observed $60$ labels, or about 60\% training labels. 
This shows they are capable of making inferences even with partial information.

\subsection{Effects of Varying Observability}
\begin{figure}[h]
\vspace{0in}
\begin{center}
\includegraphics[width=0.99\columnwidth]{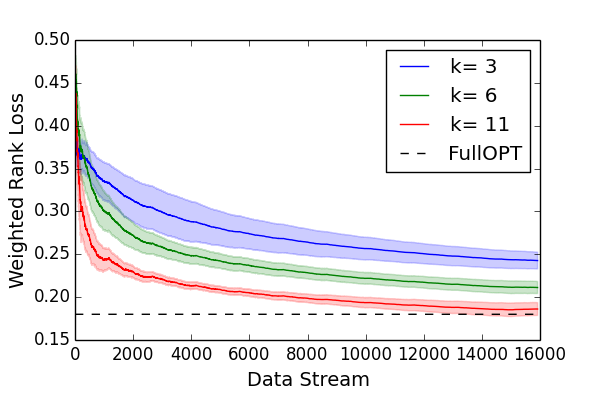}
\end{center}
\vspace{0in}
\caption{Learning Curves with 1-Standard Deviation Error Bands of \optimalAlgo~Varying $k$ on the Yeast Dataset
}
\label{fig:k}
\end{figure}


To show the empirical effects of top-$k$ feedback on model convergence and asymptotic loss, we repeat our experiments with the Yeast data set, keeping the same hyper-parameters in Table \ref{tab:hyperparameters}, but increasing $k$. Figure \ref{fig:k} plots the weighted rank loss averaged over the data stream for \optimalAlgo~models with various $k$. Each plot in the figure is the average of $10$ experiments along with a 1-standard deviation error band around the curve. We observe that as $k$ increases, the number of rounds \optimalAlgo~ requires to converge decreases. Furthermore, we observe tighter bands around the curves as $k$ increases, corroborating our theoretical results which imply smaller variance for larger $k$. 




\section{CONCLUSION}
In this paper we presented two online boosting algorithm that make multilabel ranking (MLR) predictions under top-$k$ feedback. Remarkably, we show that even under top-$k$ feedback, both algorithms enjoy the same asymptotic loss bounds as their full-information counterpart and enjoy comparable empirical performance across various datasets. There are several future directions to consider in online MLR boosting. For one, our algorithms rely heavily on pairwise decomposable loss functions, and thus, it is open to whether other types of loss functions, such as singly-decomposable loss functions, can be used for online MLR boosting under top-$k$ feedback. Second, we provide a fast, yet suboptimal, adaptive algorithm, and it would be interesting to see if an optimal adaptive algorithm is possible. 


\subsubsection*{Acknowledgements}
Part of this work occurred while DZ was an undergraduate student at the University of Michigan where he was partially supported by the NSF RTG grant DMS-1646108. This work was mostly done when YJ was at the University of Michigan. YJ and AT acknowledge the support of NSF CAREER grant IIS-1452099.

\bibliography{citations}

\clearpage
\newpage
\onecolumn
\begin{appendices}
\section{DETAILED PROOFS}
\label{sec:detailed_proofs}
We provide proofs that are skipped in the main body. 
\subsection{Proofs for the Unbiased Estimator}

\begin{lemma}
\label{unb_est_proof}
Suppose that we have $k \geq 2$, and a random ranking $\rtilde$ such that for any $a, b \in [m]$, $\Pr[a, b \in \Top^k(\rtilde)] > 0$. 
Let $g(s) \coloneqq \sum_{a\in R_t} \sum_{b \notin R_t} \f{a,b}(s)$.
Then the expectation of Eq. \ref{unb_est} becomes $g(s)$.

\begin{proof}
We first write out the unbiased estimator that Eq. \ref{unb_est} provides of $g$
$$
\hat{g}(s) =  \sum_{a\in R_t} \sum_{b \notin R_t}
    \frac{\I(a, b \in \Top^k(\rtilde))}{\Pr[a, b \in \Top^k(\rtilde)]}
    \f{a,b}(s).
$$

Then we rewrite the expectation of $\hat{g}$ by moving it inside the sum.
\begin{align*}
\E_{\rhat}[\hat{g}(s, R_t)] &=
    \sum_{a\in R_t} \sum_{b \notin R_t} \E_{\rtilde}\big[ \frac{\I(a, b \in \Top^k(\rtilde))}{\Pr[a, b \in \Top^k(\rtilde)]} \f{a,b}(s) \big] \\
    &= \sum_{a\in R_t} \sum_{b \notin R_t} \Pr[a, b \in \Top^k(\rtilde)] \frac{\f{a,b}(s)}{\Pr[a, b \in \Top^k(\rtilde)]} \\
    &= g(s)
\end{align*}
where the middle equality is due to the expectation inside the summation being zero unless $a, b \in \Top^k(\rtilde)$. 
\end{proof}
\end{lemma}

\begin{lemma}
\label{unb_est_bound}
Suppose our loss function $L$ is pairwise decomposable and thus has an unbiased estimator $\hat{L}$ If there exists $z$ such that $f^{a, b}(s) \leq z$ for all feasible $s$ available to the booster,  then under uniform randomization we have 
$$
|\hat{L}(s, R_t) - L(s, R_t)| = \mathcal{O}(z\frac{ 2m^2}{\rho}) \text{ almost surely}.
$$
\end{lemma}
\begin{proof}
We first bound the case where our estimator underestimates the true loss. In this case, the worst scenario would be all the pairwise loss functions which we activate evaluate to 0, while all other functions evaluate to $z$. In these cases, our bound is $O(z(m^2 - k^2))$ as our absolute deviation is proportional to $z$ times the number of element pairs that are not in the top-$k$. 

We now consider the case where our estimator overestimates the true loss. Let $\tilde{V} = \mathcal{T}^k(\tilde{r}_t)$, where $\tilde{r}_t$ is the randomized ranking as described above. The worst-case scenario here is when all the pairwise functions we activate evaluate to $z$, and all other pairwise functions evaluate to $0$. Note that since pairwise functions activate only if \textbf{both} $a$ and $b$ are in the top-$k$ we only really need to sum over pairs of elements in the top-$k$, as only these pairs contribute a value proportional to $z$ to the absolute deviation, while those pairs not in the top-$k$ contribute $0$. Thus, at most $k^2$ pairs need to be considered. 

Now it remains to bound $\frac{1}{Pr[a,b \in \tilde{V})]} - 1$. Note that trivially, $\frac{1}{Pr[a,b \in \tilde{V}]} - 1 \leq \frac{1}{Pr[a,b \in \tilde{V}]}$, thus it suffices to only bound the latter. For any pair of labels $a, b \in [m]$, the probability that $a,b \in \tilde{V}$ is at least the probability that $a,b$ are in the top-$k$ after possible exploration. Since we select a permutation uniformly at random when we explore, the probability that $a$ and $b$ are in the top-$k$ after exploration is simply the number of these permutations which have $a$ and $b$ in the top-$k$ over the total number of permutations $m!$ Put together,  

\[Pr[a,b \in \tilde{V}] \geq \rho \cdot \frac{2 \cdot \binom{k}{2} \cdot (m-2)!}{m!} = \rho \cdot \frac{k(k-1)}{m(m-1)}\]

The intuition behind the numerator of the fraction is that for each of the $2 \cdot \binom{k}{2}$ possible placements of $a$ and $b$ in the top-$k$, there are $(m-2)!$ ways to order the remaining labels. This leads to $\frac{1}{Pr[a,b \in \tilde{V}]} - 1 \leq \frac{m^2}{\rho k(k-1)}$. Finally, putting together the number of iterations and the above probability bound, we find that when our estimator overestimates the true loss in the worst case by:
     \begin{align*}
        |\hat{L}(s, R_t) - L(s, R_t)| &\leq z \frac{k^2m^2}{\rho k(k-1)} \\
        &= \frac{k}{k-1} \cdot z\frac{m^2}{\rho} \\
        &\leq z \frac{2m^2}{\rho} 
    \end{align*}
where the second inequality follows from the assumption that $k \geq 2$. Thus, overall $ |\hat{L}(s, R_t) - L(s, R_t)| = O(z\frac{2m^2}{\rho})$.

\end{proof}

\subsection{Proofs for the Optimal Algorithm}
\label{sec:optimal_proofs}
For our optimal algorithm proofs, we require an important lemma comparing biased uniform distributions and our surrogate potential function distributions.

\begin{lemma}
\label{recurs_bound}
Let $a, b \in [m]$ and $\f{a, b}(\cdot)$ be a pairwise function which satisfies the three assumptions stated in Section \ref{sec:potential}. 
Then for any set of relevant labels $R \subset [m]$, we have
$$
\E_{e_l \sim u^\gamma_R}[\f{a,b}(s+e_l)] \leq 
    \E_{e_{l^\prime} \sim u^\gamma_a}[\f{a,b}(s+e_{l^\prime})].
$$
where $u^\gamma_R$ and $u^\gamma_a$ are the biased uniform distributions, placing $\gamma$ more weight on members of $R$ and the label $a$ respectively.
\begin{proof}
Recall $\f{a,b}(s)= 
\I(a \in R) \I(b \notin R)f(s[a], s[b])$.
Hence, if $a \notin R$ or $b \in R$, then $\f{a,b}$ becomes a zero function, and the inequality trivially holds. 

Suppose $a \in R$ and $b \notin R$. 
By definition of $u^\gamma_a$ and $u^\gamma_R$, we have 
\[
u^\gamma_a[a] - u^\gamma_a[b] 
=
u^\gamma_R[a] - u^\gamma_R[b]
= \gamma,
\]
from which we can deduce 
\[
u^\gamma_a[a] - u^\gamma_R[a]
=
u^\gamma_a[b] - u^\gamma_R[b]
=: \Delta > 0
\text{ and }
\sum_{l\in[m] - \{a, b\}}
    u^\gamma_a[l] - u^\gamma_R[l]
=
-2\Delta
.
\]
Furthermore, observe that if $l \notin \{a, b\}$, then 
$\f{a, b}(s) = \f{a, b}(s + e_l)$. 
From this, we can infer
\begin{align*}
\E_{e_{l^\prime} \sim u^\gamma_a}[\f{a,b}(s+e_{l^\prime})] 
-
\E_{e_l \sim u^\gamma_R}[\f{a,b}(s+e_l)]
&=
\sum_{l \in [m]}(u^\gamma_a[l] - u^\gamma_R[l]) \f{a, b}(s+e_l) \\
&=
\Delta \cdot\bigg(
\f{a, b}(s+e_a)
+
\f{a, b}(s+e_b)
-
2\f{a, b}(s)
\bigg).
\end{align*}
Using the uncrossability and convexity of $f$, we can show
\[
\f{a, b}(s+e_a)
+
\f{a, b}(s+e_b)
\ge
2\f{a, b}(s+\frac{1}{2}(e_a +e_b))
\ge
2\f{a, b}(s),
\]
which finishes the proof.

\end{proof}
\end{lemma}





\subsubsection{Proof of Proposition~\ref{mylemma}}
\primelemma*

\label{sec:pot_bound_proof}
\begin{proof}
We note that $\Phi$ is the sum of many pairwise potential functions, each of which uses a convex and proper pairwise function. By nature of potential functions, these smaller potentials must be proper and convex thus $\Phi$ must be as well.

To prove the upper bound, we first expand $\Upsilon$ and $\Phi$ as the sum over pairs of potential functions:
\begin{align*}
\Upsilon^N_t(s) &= \sum_{a \in R_t} \sum_{b \notin R_t}
    \varphi^N_{u^\gamma_{R_t}}(s, \f{a,b}) \\
\Phi^N_t(s) &= \sum_{a \in R_t} \sum_{b \notin R_t}
    \Lambda^{a,b,N}_t(s).
\end{align*}

Lemma \ref{recurs_bound} implies
$$
\varphi^1_{u^\gamma_{R_t}}(s, \f{a, b}) 
\leq 
\Lambda^{a, b, 1}_t(s).
$$
We also note that $\Lambda^{a, b, N}_t(s)$ is also proper, uncrossable, and convex in its scores, and thus
$$
\varphi^2_{u^\gamma_{R_t}}(s, \f{a, b}) 
\leq 
\Lambda^{a, b, 2}_t(s).
$$
We can repeat this process recursively to obtain for any $N$ that
$$
\varphi^N_{u^\gamma_{R_t}}(s, \f{a,b}) \leq \Lambda^{a,b,N}_t(s)
$$
and finally sum across all pairs in $R_t \times R^c_t$ to obtain the desired inequality.
\end{proof}

\subsection{Proof of Theorem~\ref{bbm_bound}}
\label{opt_proof}
\bbmbound*
\begin{proof}
We first show a recurrence relation involving $\Lambda$.
\begin{align*}
\Lambda^{a, b, N-i}_{t}(s^i_t) &= \varphi^{N-i}_{u^\gamma_a}(s^i_t, \f{a, b}) \\
&\geq \E_{e_l \sim u^\gamma_{R_t}}\varphi^{N-i-1}_{u^\gamma_a}
(s^i_t + e_l, \f{a, b}) \\
    &= \E_{e_l \sim u^\gamma_{R_t}}\Lambda^{a, b, N-i-1}_t(s^{i}_t+e_l),
\end{align*}
where the inequality holds from Lemma 7.
Then, summing across all pairs of elements, we have that
\begin{align*}
\Phi^{N-i}_{t}(s^i_t) &= \sum_{a \in R_t} \sum_{b \notin R_t} \Lambda^{a, b, N-i}_{t}(s^{i}_t) \\ 
&\geq \E_{e_l \sim u^\gamma_{R_t}} \bigg[ \sum_{a \in R_t} \sum_{b \notin R_t}
\Lambda^{a, b, N-i-1}_t(s^{i}_t+e_l)
\bigg] \\
&= c^{i+1}_t \cdot u^\gamma_{R_t} \\
&= c^{i+1}_t \cdot (u^\gamma_{R_t} - h^{i+1}_t) + c^{i+1}_t \cdot h^{i+1}_t \\
&\geq c^{i+1}_t \cdot (u^\gamma_{R_t} - h^{i+1}_t) + \Phi^{ N-i-1}_{t}(s^{i+1}_t) 
\end{align*}
where the last inequality holds due to the convexity of $\Phi$ and Jensen's inequality.

Then summing over $t$, we have that
$$
\sum_t \Phi^{N-i}_{t}(s^i_t) \geq \sum_t c^{i+1}_t \cdot (u^\gamma_{R_t} - h^{i+1}_t) + \Phi^{ N-i-1}_{t}(s^{i+1}_t).
$$
and with probability at least $1-\delta$, the $\WLC(\gamma, \delta, S)$ holds, and we have that
$$
\sum_t \Phi^{N-i}_t(s^i_t) -  \Phi^{ N-i-1}_{t}(s^{i+1}_t) \geq -S.
$$

Then summing across all $N$, by the telescoping rule we have that
$$
\sum_t \Phi^N_{t}(\textbf{0}) + NS \geq \sum_t \Phi^0_{t}(s^N_t) = \sum_t L(s^N_t, R_t).
$$
Simplifying and noting that $S = \Otilde(\frac{2m^2}{\rho}z)$ proves the theorem.
\end{proof}

\subsection{Proof of Corollary \ref{opt_corollary}}
\label{opt_corollary_proof}
We first prove a lemma regarding potential functions and the hinge loss. 
It modifies Lemma 8 by \citet{jungMultilabel}.
\begin{lemma}
\label{pot_zero_bound}
With the unweighted hinge loss as the loss function, we have that
$$
\Phi^N_t(0) \leq (N+1)\frac{m^2}{4}\exp(-\frac{\gamma^2N}{2}).
$$
\begin{proof}
For convenience, we drop the subscript $t$ from this proof. Recalling that $\Phi^N$ is the sum of many $\Lambda^{a, b, N}$, where $a \in R_t, b \notin R_t$, we bound each $\Lambda$ on its own first. Let $X^N$ be the result of the random walk from the potential function $\Lambda$ is defined with, sampling $N$ times from $u^\gamma_a$ and adding the result to $0$. Then we can rewrite $\Lambda$ as
\begin{equation}
\label{lamda_bound_inside}
\begin{aligned}
\Lambda^{a, b, N}(0) &= 
    \E_{X^N}[\max\{0, X^N[b] - X^N[a]\}] \\
    &= \sum_{n=0}^N \Pr[X^N[b] - X^N[a] \geq n] \\
    &\leq (N+1)\Pr[X^N[b] - X^N[a] \geq 0].
\end{aligned}
\end{equation}
Next, we define the probabilities $p \coloneqq u^\gamma_a[a], q \coloneqq u^\gamma_a[b]$, and we interpret $\Pr[X^N[b] - X^N[a] \geq 0]$ using a game, where with probability $p$ we draw $-1$, while with probability $q$ we draw $1$. Then $\Pr[X^N[b] - X^N[a] \geq 0]$ equals the probability that the summation of $N$ i.i.d.\ random numbers are non-negative. Thus we can apply Hoeffding's inequality to obtain
$$
\Pr[X^N[b] - X^N[a] \geq 0] \leq \exp(-\frac{\gamma^2N}{2})
$$
and plugging back into Eq. \ref{lamda_bound_inside} to obtain
$$
\Lambda^{a, b, N}(0) \leq (N+1)\exp(-\frac{\gamma^2N}{2}).
$$
Lastly, we recall that $\Phi^N_t$ sums over every pair of relevant and irrelevant label using $\Lambda$. This requires us to note that the maximum possible number of such pairs is $\frac{m^2}{4}$, when $\frac{m}{2}$ of the labels are relevant. Multiplying $\Lambda^{a,b,N}(0)$ by the maximum possible number of pairs provides our bound.
\end{proof}
\end{lemma}

Now we are ready to prove Corollary~\ref{opt_corollary}.
\optcorollary*
\begin{proof}
Firstly, we note that if we decide to explore, then the number of additional pairs that could be dis-ordered is no greater than $m^2$. Noting that the worst case scenario upper bounds the expected loss from exploration, we have the inequality
$$
\sum_{t=1}^T \E_{\ytilde_t}[\Rank(\ytilde_t, R_t)] - \Rank(s^N_t, R_t) \leq \rho m^2T.
$$

Then, we note that the random additional unweighted rank loss we obtain from exploration is bounded each round by $m^2$ because this upper bounds the unweighted rank loss.
We apply concentration inequalities and rearrange, so that with probability at least $1-\delta$, an upper bound on the deviation of our true exploration loss from the expected exploration loss is
$$
\sum_{t=1}^T \Rank(\ytilde_t, R_t) - \E_{\ytilde_t}[\Rank(\ytilde_t, R_t)] \leq m^2\sqrt{T\log\frac{1}{\delta}}.
$$

Combining these inequalities and plugging into Theorem \ref{bbm_bound} 
we have
\begin{align*}
\sum_t \Rank(\ytilde_t, R_t) &\leq \sum_t \Rank(s^N_t, R_t) + \rho m^2T + m^2 \sqrt{T\log\frac{1}{\delta}} \\
&\leq \Phi^N_t(0)T + \Otilde(\frac{2m^2 N^2}{\rho}) +\rho m^2T + m^2 \sqrt{T\log\frac{1}{\delta}}
\end{align*}
where in obtaining the expression inside $\Otilde$ we use that each pairwise function in the unweighted hinge loss is upper bounded by $N$.

Finally, using Lemma \ref{pot_zero_bound} we can produce our final bound.

\begin{align*}
\sum_t \Rank(\ytilde_t, R_t)
&\leq \frac{m^2}{4}(N+1)\exp(-\frac{\gamma^2 N}{2})T + \rho m^2 T  + \Otilde(\frac{2m^2}{\rho}(N^2 + \sqrt{T})) \\
&\leq \frac{m^2}{4}(N+1)\exp(-\frac{\gamma^2 N}{2})T + \rho m^2 T + \Otilde(\frac{2m^2}{\rho}N^2\sqrt{T}).
\end{align*}

\end{proof}

\begin{remark}
The factor of $m^2$ when accounting for the number of additional pairs that could be dis-ordered due to uniform randomization can be improved to a factor of $m$ by using an alternative randomization scheme. See Appendix \ref{sec:alt_rand} for a more detailed discussion of the trade-offs between different randomization schemes. 
\end{remark}

\subsection{Proof of Theorem~\ref{ada_bound}}
\label{sec:adaptiveProof}
\adabound*
\begin{proof}We start by defining the unbiased estimate  and true unweighted rank loss suffered by the
$i$th expert as
\begin{gather*}
\Mhat_i = \sum_t \hatrankloss(s^i_t, R_t) 
\text{ and }
M_i = \sum_t \Rank(s^i_t, R_t).
\end{gather*}
Recall that our unbiased estimator for unweighted rank loss is bounded by 
$\Otilde(\frac{2m^2}{\rho})$.
We let $\Mhat_0 = \sum_t \Lhat(0, R_t)$ and $M_0 = \sum_t \Rank(0, R_t)$. If we write $i^* = \argmin_i M_i$, then we have by
the concentration inequality, with probability at least $1-\delta$ that
$$
\min_i \Mhat_i \leq \Mhat_{i^*} \leq \min_i M_i + \Otilde(\frac{2m^2}{\rho})\sqrt{T\log\frac{1}{\delta}} = \min_i M_i + \Otilde(\frac{2m^2}{\rho}\sqrt{T}).
$$
Because the booster chooses an expert through the Hedge algorithm and because we feed the hedge algorithm $\Mhat_i$ each round, a standard analysis as in Corollary 2.3 of \cite{hedge_analysis} gives that
\begin{gather}
\label{hedge}
\sum_t \Rank(\yhat_t, R_t) \leq \sum_t \hatrankloss(\yhat_t, R_t) + \Otilde(\frac{2m^2}{\rho}\sqrt{T}) \leq 2 \min_i M_i + 2\log N + \Otilde(\frac{2m^2}{\rho}\sqrt{T}).
\end{gather}

where the inequality between $\Rank$ and $\hatrankloss$ is again using concentration inequalities, and with another probability at least $1-\delta$.

Now we check that $\frac{1}{1+\exp(a-b)} \geq \frac{1}{2}\I(a \leq b)$, so that
\begin{equation}
\label{w_bound}
\begin{aligned}
w^i[t] &\geq \frac{1}{2}\Rank(s^{i-1}_t, R_t) 
\text{ and }
\norm{w^i}_1  &\geq \frac{M_{i-1}}{2}
\end{aligned}
\end{equation}
with $w^i[t]$ defined as in Eq. \ref{edge_define}.

Now we let $\Delta_i$ denote the difference between cumulative logistic loss
between two experts
\begin{align*}
\Delta_i &= \sum_t \Logloss(s^i_t, R_t) - \Logloss(s^{i-1}_t, R_t) \\
    &= \sum_t \Logloss(s^{i-1}_t + \alpha^i_t h^i_t, R_t) - \Logloss(s^{i-1}_t, R_t).
\end{align*}
Then a standard analysis of stochastic gradient descent \citep{SGD} provides that, with probability at least
$1-\delta$
\begin{equation}
\label{sgd}
\begin{aligned}
\Delta_i \leq &\min_{\alpha \in [-2, 2]}
\sum_t \bigg[ \Logloss(s^{i-1}_t + \alpha h^i_t, R_t) - \Logloss(s^{i-1}_t, R_t) \bigg]
+ \Otilde(\frac{2m^2}{\rho}\sqrt{T})
\end{aligned}
\end{equation}

We now record a useful inequality
\begin{align*}
\log(1 + e^{s+\alpha}) - \log(1 + e^s) &= \log(1 + \frac{e^\alpha-1}{1+e^{-s}}) \leq \frac{1}{1+e^{-s}}(e^\alpha - 1).
\end{align*}

Using this we can expand
\begin{align*}
\sum_t \Logloss(s^{i-1}_t + \alpha h^i_t, R_t) - \Logloss(s^{i-1}_t, R_t) 
    &= \sum_t \sum_{a \in R_t} \sum_{b \notin R_t} 
        \log \frac{1 + \exp(s^{i-1}_t[b] - s^{i-1}_t[a] + \alpha(h^i_t[b] - h^i_t[a]))}{1+ \exp(s^{i-1}_t[b] - s^{i-1}_t[a])} \\
    &\leq \sum_t \sum_{a \in R_t} \sum_{b \notin R_t} \frac{\exp(\alpha(h^i_t[b] - h^i_t[a]))-1}{1+ \exp(s^{i-1}_t[a]-s^{i-1}_t[b])} \coloneqq f(\alpha)
\end{align*}
We also rewrite $\norm{w^i}_1$ and $\gamma_i$ as the following:
\begin{equation}
\label{redef_w_gamma}
\begin{aligned}
\norm{w^i}_1 &= \sum_t \sum_{a \in R_t} \sum_{b \notin R_t}
    \frac{1}{1+ \exp(s^{i-1}_t[a] - s^{i-1}_t[b])} \\
\gamma_i =& \sum_t \sum_{a \in R_t} \sum_{b \notin R_t}
    \frac{1}{\norm{w^i}_1} \frac{h^i_t[a] - h^i_t[b]}{1+ \exp(s^{i-1}_t[a] - s^{i-1}_t[b])}.
\end{aligned}
\end{equation}
For ease of notation, let $j$ be an index which loops over all tuples $(t, a, b)$ $\in$ $[T] \times R_t \times R^c_t$,
and let $a_j$ and $b_j$ be the following terms:
\begin{align*}
a_j &= \frac{1}{\norm{w^i}_1}\frac{1}{1+ \exp(s^{i-1}_t[a] - s^{i-1}_t[b])} \\
b_j &= h^i_t[a] - h^i_t[b].
\end{align*}
Then from Eq. \ref{redef_w_gamma} we have that $\sum_j a_j = 1$ and $\sum_j a_j b_j = \gamma_i$. Then, we can express $f(\alpha)$ using $a_j$ and $b_j$ as
\begin{align*}
\frac{f(\alpha)}{\norm{w^i}_1} &= \sum_j a_j(\exp(-\alpha b_j) - 1) \\
    &\leq \exp(-\alpha \sum_j a_j b_j) - 1  \\
    &= \exp(-\alpha \gamma_i) - 1
\end{align*}
where the inequality holds by Jensen's. From this, we can deduce that
\begin{gather}
\label{conv_opt}
\min_{\alpha \in [-2, 2]} f(\alpha) \leq
    -\frac{|\gamma_i|}{2}\norm{w^i}_1.
\end{gather}

Then combining equations \ref{w_bound}, \ref{sgd}, and \ref{conv_opt}, we have
$$
\Delta_i \leq -\frac{|\gamma_i|}{4}M_{i-1} + \Otilde(\frac{2m^2}{\rho}\sqrt{T}).
$$
Summing over $i$, we get by telescoping rule
\begin{align*}
\sum_t \Logloss(s^N_t, R_t) - \Logloss(0, R_t) 
&\leq 
-\frac{1}{4}\sum_{i=1}^N |\gamma_i|M_{i-1} + \Otilde(\frac{2m^2}{\rho}N\sqrt{T}) \\
&\leq
-\frac{1}{4}\sum_{i=1}^N |\gamma_i|\min_i M_{i} + \Otilde(\frac{2m^2}{\rho}N\sqrt{T}).
\end{align*}
Then, solving for $\min_i M_i$, we obtain
\begin{align*}
\min_i M_i &\leq 
\sum_t \frac{4\Logloss(0, R_t)}{\sum_i |\gamma_i|} + \Otilde(\frac{2m^2N}{\rho \sum_i |\gamma_i|}\sqrt{T}) \\
&\leq
\frac{m^2 \log 2}{\sum_i |\gamma_i|}T + \Otilde(\frac{2m^2N}{\rho \sum_i |\gamma_i|}\sqrt{T}),
\end{align*}
where we use that in $\Logloss(0, R_t)$, each pairwise function must evaluate to $\log 2$, and there can be at most
$\frac{m^2}{4}$ pairwise functions. Plugging this result into Eq. \ref{hedge}, we have
$$
\sum_t \Rank(\yhat_t, R_t) \leq
\frac{2 m^2 \log 2}{\sum_i |\gamma_i|}T + \Otilde(\frac{2m^2N}{\rho \sum_i |\gamma_i|}\sqrt{T}).
$$

Finally, since we predict with $\ytilde_t$ instead of $\yhat_t$, we must add in the loss from exploration. We follow the same steps as in the proof of Corollary \ref{opt_corollary} to obtain
$$
\sum_t \Rank(\ytilde_t, R_t) \leq
\frac{2 m^2 \log 2}{\sum_i |\gamma_i|}T
+ \rho m^2 T +
\Otilde(\frac{2m^2N}{\rho \sum_i |\gamma_i|}\sqrt{T}).
$$

\end{proof}

\subsection{Additional Proofs}
\label{loss_properties_proof}
\begin{lemma}
The logistic and hinge unweighted are pairwise decomposable into functions which are proper, convex, and uncrossable.
\begin{proof}
We first note that for any $a \in R, b \notin R$, we can write the logistic and hinge losses as
\begin{gather*}
\begin{aligned}     
\Logloss(s, R) &= \sum_{a \in R} \sum_{b \notin R}
    \log(1 + \exp(s[b] - s[a])), \text{ and} \\
\hinge(s, R) &= \sum_{a \in R} \sum_{b \notin R}
    \max\{0, 1 + s[b] - s[a]\}.
\end{aligned}
\end{gather*}
In both cases, adding a constant scalar to the scores for both $a$ and $b$ will cancel itself out, thus satisfying the uncrossability. 
Also, the fact that the functions inside the summations are proper and convex finishes the proof.
\end{proof}
\end{lemma}

\section{ALTERNATE RANDOMIZATION PROCEDURES}
\label{sec:alt_rand}
While the uniform exploration scheme is simple and has empirically shown to perform well, the induced additional error on the rank loss and the $b$-boundedness of its loss estimator are theoretically sub-optimal. Here, we give an alternative randomization procedure that improves upon the bounds of the uniform exploration scheme, but is more sophisticated and requires that $k \geq 3$.

\subsection{Single-Swap Randomization}

Consider the following single-swap randomization scheme: after computing $r_t$, with probability $1-\rho$ we use $r_t$ as our final ranking. Otherwise, with probability $\rho$, we choose one label from $\mathcal{T}^k(r_t)$ and one label from $\mathcal{T}^k(r_t)^c$ (the set of labels which have rank lower than k) and swap them. Let $\hat{r}_t$ be the randomized ranking after a single swap. Then, we repeat this process once more on $\hat{r}_t$,  selecting one element from $\mathcal{T}^k(\hat{r}_t)$ and one from $\mathcal{T}^k(\hat{r}_t)^c$ and swapping them. Let $\tilde{r}_t$ represent the final randomized ranking after this last step. 

\begin{lemma}
Suppose our loss function $L$ is pairwise decomposable, $\rho < 0.25$, and thus has an unbiased estimator $\hat{L}$ If there exists $z$ such that $f^{a, b}(s) \leq z$ for all feasible $s$ available to the booster,  then under the single-swap randomization we have 
$$
|\hat{L}(s, R_t) - L(s, R_t)| = \mathcal{O}(z\frac{ 2m^2-k^2}{\rho}) \text{ almost surely}.
$$
\end{lemma}


\begin{proof}
We first bound the case where our estimator underestimates the true loss. In this case, the worst scenario would be all the pairwise loss functions which we activate evaluate to 0, while all other functions evaluate to $z$. In these cases, our bound is $O(z(m^2 - k^2))$ as our absolute deviation is proportional to $z$ times the number of element pairs that are not in the top-k. 

We now consider the case where our estimator overestimates the true loss. Let $V = \mathcal{T}^k(r_t)$ and $\tilde{V} = \mathcal{T}^k(\tilde{r}_t)$, where $\tilde{r}_t$ is the \textit{final} randomized ranking as described above. Let $\hat{V} = \mathcal{T}^k(\hat{r}_t)$. We now separately bound $\frac{1}{Pr[a,b \in \tilde{V}]} - 1$ for the cases when (1) $a,b \in V$ (2) one of $a$ and $b$ is in $V$ and (3) when $a,b \notin V$. Notice that trivially, $\frac{1}{Pr[a,b \in \tilde{V}]} - 1 \leq \frac{1}{Pr[a,b \in \tilde{V}]}$, and thus for cases (2) and (3) we only bound $\frac{1}{Pr[a,b \in \tilde{V}]}$.


Suppose $a,b \in V$. Then since we only explore with probability $\rho \leq 0.25$, and $\frac{1}{Pr[a,b \in \tilde{V}]} \geq 1 - \rho$ and  $\frac{1}{Pr[a,b \in \tilde{V}]} - 1 \leq \frac{1}{3}$. In addition, there must be fewer than $k^2$ of these pairs.

Secondly, suppose exactly one of $a$ and $b$ is in $V$ . Without loss of generality, let us assume $a \in V$ and $b \notin V$. Then, for both $a$ and $b$ to be in $\tilde{V}$, the algorithm must decide to explore with probability $\rho$. Then, one possible way of bringing $b$ into $\tilde{V}$ is by swapping $b$ with a label $c (\neq a) \in V$ in the first round, and then swapping a label $d (\neq a, b) \in \hat{V}$ with any label from $\hat{V}^c$. From this, we obtain,

\[Pr[a,b \in \tilde{V}] \geq \rho \cdot \frac{k-1}{k} \cdot \frac{1}{m-k} \cdot \frac{k-2}{k}\]

\noindent Therefore, 

\[\frac{1}{Pr[a,b \in \tilde{V}]} - 1 \leq \frac{k^2(m-k)}{\rho(k-2)(k-1)}\]

\noindent Here, the upper bound on the total number of such pairs is $2(k-2)$, as over the two rounds of swapping, at most two different labels can be transplanted from outside the top-$k$. However, note that if $k = 2$, this case can never occur, as after the second swap it is guaranteed that one of $a$ or $b$ is swapped out of $\hat{V}$. This is why the following scheme requires that $k \geq 3$.

Lastly, suppose $a, b \notin V$ . They both must be chosen to be moved up when the algorithm decides to explore with
probability $\rho$. Thus, in the first round of randomization, we must choose either $a$ or $b$, and then, in the second round, we must not select the label that was chosen in the first round and must select the label that was not chosen in the first round. Formally, 

\[Pr[a,b \in \tilde{V}] = \rho \cdot \frac{2}{m-k} \cdot \frac{1}{m-k} \cdot \frac{k-1}{k}\]

Which implies, 

\[\frac{1}{Pr[a,b \in \tilde{V}]} - 1 \leq \frac{k(m-k)^2}{\rho(k-1)}\]

There at most can only be one of these pairs present, because we only ever choose at most two different labels from outside of the top-$k$ to swap. Now, to produce our final bound, we multiply the weight produced from each case by the number of times it can occur, to obtain the sum:

     \begin{align*}
        G &\leq z\left[ \frac{1}{3}k^2 + 2(k-2)\frac{k^2(m-k)}{\rho(k-2)(k-1)} + \frac{k(m-k)^2}{\rho(k-1)}\right] \\
        &= z\left[ \frac{1}{3}k^2 + \frac{2k^2(m-k)}{\rho(k-1)} + \frac{k(m-k)^2}{\rho(k-1)}\right] \\
        &= z\left[ \frac{1}{3}k^2 + \frac{k}{k-1}\cdot \frac{m^2 - k^2}{\rho}\right] \\
        &\leq z\left[ \frac{1}{3}k^2 +  \frac{2(m^2 - k^2)}{\rho}\right] \\
        &= z\left[ \frac{2(m^2 - k^2) + \frac{1}{3}k^2\rho}{\rho} \right] \\
        &\leq z\left[ \frac{2(m^2 - k^2) + \frac{1}{3}k^2}{\rho} \right] \leq z\left[ \frac{2m^2 - k^2}{\rho} \right] = \mathcal{O}(z\frac{2m^2 - k^2}{\rho}) \\
    \end{align*}
    
The middle inequality follows from the restriction that $k \geq 2$ and the last inequality follows from $\rho \leq 0.25$.  

\end{proof}

Note, that this bound improves upon the $b$-boundedness of the uniform randomization scheme, but requires that $k \geq 3$. 
In addition, unlike the uniform exploration scheme, if we decide to explore using the single-swap scheme, then the number of additional label pairs, and thus the additional rank loss, that could be disordered due to exploration is no greater than $2m$. To see this, we observe the worst case scenario for exploration, where the the top ranked label is relevant but is being swapped with the lowest ranked label, which is irrelevant. Immediately this provides our first newly disordered pair. Let $M$ be the set of labels not ranked first or last, in other words, all the other labels. For every irrelevant label in $M$, another pair associated with the formerly top-ranked and relevant label will become incorrect. Then, for every relevant label in $M$, another pair associated with the formerly lowest-ranked and irrelevant label will become incorrect after swapping. Thus for every member of $M$, an additional pair is disordered by the swap, which is upper bounded by $m$. Thus, the total number of pairs disordered is $|M|+1 = m-1 < m$. Finally, multiplying by the number of swaps gives us the upperbound on the number of disordered pairs. 

Although the single-swap scheme poses restrictions on $k$ and $m$, its improvement on the $b$-boundedness of the unbiased loss estimator and reduction in worst-case loss due to exploration directly implies better final loss bounds for the booster. More specifically, by using the single-swap randomization scheme, it is not hard to see that the final loss bound in the optimal setting is

\begin{align*}
\sum_t \Rank(\ytilde_t, R_t)
&\leq \frac{m^2}{4}(N+1)\exp(-\frac{\gamma^2 N}{2})T + 2\rho m T + \Otilde(\frac{2m^2-k^2}{\rho}N^2\sqrt{T})
\end{align*}

and the final loss bound in the adaptive setting is:

$$
\sum_t \Rank(\ytilde_t, R_t) \leq
\frac{2 m^2 \log 2}{\sum_i |\gamma_i|}T
+ 2\rho m T +
\Otilde(\frac{(2m^2-k^2)N}{\rho \sum_i |\gamma_i|}\sqrt{T}).
$$

These can be directly verified by following the proofs presented in Appendix \ref{sec:detailed_proofs}, but instead using the values corresponding to the single-swap randomization scheme. We can see that compared to the loss bounds of the uniform randomization scheme, the factor of $m^2$ of the middle term is reduced to a factor of $m$, and the excess loss bound shows a more favorable dependence on $k$.


\subsection{Experiments}

To compare the performance of using single-swap randomization scheme to the uniform randomization scheme, we rerun the same experiments described in section 4 using single-swap randomization. Table \ref{tab:uniform_single_swap} gives the weighted average rank loss using single-swap randomization after optimizing hyperparameters. Here, \say{-S} refers to the single-swap randomization scheme and \say{-U} refers to the uniform randomization scheme.  We include again the results of the full information algorithm (denoted by prefix \say{Full}) for completeness sake. 
    
\begin{table}[t]
\centering
\caption{Average Weighted Rank Loss on Test Set using Single-Swap Randomization}
    \setlength\tabcolsep{2pt}
    \vspace{10pt}
    \begin{tabular}{cccccccc}
    \toprule
    Data &TopOpt-S &TopAda-S &TopOpt-U &TopAda-U &FullOpt &FullAda\\
    \midrule
    Emotions                  &0.19  &0.23 &0.20  &0.22   &0.17   &0.16\\ 
    Scene                   &0.11   &0.13 &0.11   &0.13   &0.07   &0.09 \\
    Yeast                     &0.22   &0.22  &0.23   &0.23   &0.18   &0.19 \\
    Mediamill                &--   &0.08 &--   &0.09   &--   &0.05 \\
    M-reduced          &0.11   &0.11 &0.11   &0.11   &0.06   &0.06\\
    \bottomrule
    \end{tabular}
    \label{tab:uniform_single_swap}
\end{table}

 Based on Table \ref{tab:uniform_single_swap}, we find that the average losses between the single-swap and uniform randomization schemes are very close across all datasets. Thus, while the single-swap randomization scheme enjoys better theoretical guarantees, its empirical advantage to the uniform randomization scheme are small.

\section{EXPERIMENT DETAILS}
\label{sec:experiment_details}
\subsection{Additional Metrics}

\begin{table}[t]
\centering
\caption{Mean and Standard Deviations of Average Weighted Rank Loss on Test Set}
    \setlength\tabcolsep{2pt}
    \vspace{10pt}
    \begin{tabular}{cccc}
    \toprule
    Data &$m$  &TopOpt &TopAda \\
    \midrule
    Emotions            &6       &$0.20\pm 0.013$  & $0.22 \pm 0.031$  \\
    Scene               &6      &$0.11 \pm 0.010$   &$0.13 \pm0.015$  \\
    Yeast               &14      &$0.23 \pm 0.0054$   &$0.23\pm0.0059$  \\
    Mediamill           &101     &--   &$0.092 \pm0.0011$ \\
    M-reduced   &101       &$0.11 \pm0.0087$   &$0.11 \pm0.0069$  \\
    \bottomrule
    \end{tabular}
    \label{tab:std}
\end{table}

\begin{table}[t]
\centering
\caption{Average Runtime in Seconds using Uniform Randomization on Test Set}
    \setlength\tabcolsep{2pt}
    \vspace{10pt}
    \begin{tabular}{cccc}
    \toprule
    Data &$m$  &TopOpt &TopAda \\
    \midrule
    Emotions            &6       &299  &189  \\
    Scene               &6      &415   &371  \\
    Yeast               &14      &334   &260  \\
    Mediamill           &101     &--   &4752  \\
    M-reduced   &101       &1852   &352  \\
    \bottomrule
    \end{tabular}
    \label{tab:runtime}
\end{table}

Tables \ref{tab:std} and \ref{tab:runtime} gives the standard deviation and average runtimes for both \optimalAlgo~ and \adaptiveAlgo~ algorithms on all datasets. We observe that for each dataset, the average runtime of the adaptive algorithm is significantly faster than its optimal counterpart. 

\subsection{Procedure}
We used the VFDT algorithms presented in \cite{VFDT} as the weak learners. Every algorithm used 100 trees whose parameters were randomly chosen. VFDT is trained using single-label data, and we fed individual relevant labels along with importance weights that were computed as $\max_{l} \hat{c}_t - \hat{c}_t[l]$. Instead of using all covariates, the booster fed to trees randomly chosen 20 covariates to make weak predictions less correlated. This procedure matches what was done by \cite{jungMultilabel}.

\begin{table}[ht]
    \centering
    \setlength\tabcolsep{20pt}
    \caption{Parameters for Algorithms in Format \textit{TopOpt/AdaOpt}}
    \vspace{10pt}
    \begin{tabular}{cccc}
        \toprule
         Dataset    &\# Weak Learners         &$\rho$             &\# Loops \\
         \midrule
         Emotions   &$50$/$50$      &$0.02$/$0.02$      &$20$/$10$ \\
         Scene      &$50$/$50$      &$0.02$/$0.04$      &$10$/$10$ \\
         Yeast      &$30$/$60$      &$0.03$/$0.04$      &$10$/$10$ \\
         Mediamill  &$10$/$10$      &$0.02$/$0.06$      &$20$/$20$ \\
         M-reduced  &$20$/$60$      &$0.04$/$0.06$      &$20$/$20$ \\
         \bottomrule
    \end{tabular}
    \label{tab:hyperparameters}
\end{table}

\begin{table}[ht]
    \centering
    \caption{Summary of Datasets}
    \vspace{10pt}
    \setlength\tabcolsep{10pt}
    \begin{tabular}{cccccccc}
        \toprule
        Dataset     &\# train &\# test  &dim    &$m$    &min    &mean   &max\\
        \midrule
        Emotions       &391        &202    &72     &6      &1      &1.87   &3\\
        Scene          &1211       &1196   &294    &6      &1      &1.07   &3\\
        Yeast          &1500       &917    &103    &14     &1      &4.24   &11\\
        Mediamill      &30993      &12914  &120    &101    &0      &4.38   &18\\
        M-reduced      &1500       &500    &120    &101    &0      &4.39   &13\\
        \bottomrule
    \end{tabular}
    \label{tab:dataset_summary}
\end{table}

We set loops to be either $10$ or $20$, optimize $N$ with granularity down to multiples of $10$, and the edge $\gamma$ for TopOpt from the set $\{0.1, 0.2, 0.3, 0.4\}$. One additional heuristic we added to the adaptive algorithm is to clip the gradient estimates that were of magnitude greater than $1.0$, down to either $1.0$ or $-1.0$ based on their sign. We found this improved convergence of SGD. We also clipped all computed probabilities to the range [0.005, 0.995].

\end{appendices}

\end{document}